\newif\ifllncs\llncsfalse
\newif\ifanon\anontrue

\ifllncs
\documentclass{llncs}
\else
\documentclass[11pt,letter]{article}
\fi 

\usepackage{breakcites}
\usepackage[utf8]{inputenc}
\usepackage[T1]{fontenc}
\ifllncs
\else
\usepackage{times}
\usepackage{fullpage}
\fi 

\usepackage{mdframed}
\usepackage{amsmath}
\usepackage{relsize}
\usepackage{amsfonts}
\usepackage{amssymb}
\usepackage[dvipsnames]{xcolor}
\usepackage{graphicx}
\usepackage{braket}
\usepackage{url}
\usepackage{bm}
\usepackage{tikz}
\usetikzlibrary{shapes,arrows,positioning,calc}
\usepackage[mathscr]{euscript}
\allowdisplaybreaks

\usepackage{soul}
\usepackage{multirow}
\definecolor{DarkBlue}{RGB}{0,0,150}
\definecolor{NotSoDarkBlue}{RGB}{15,15,210}
\definecolor{DarkRed}{RGB}{150,0,0}
\definecolor{DarkGreen}{RGB}{0,100,0}
\usepackage[pdfstartview=FitH,pdfpagemode=UseNone,colorlinks,linkcolor=DarkRed,filecolor=blue,citecolor=DarkRed,urlcolor=DarkRed,pagebackref,breaklinks]{hyperref}
\usepackage[ruled, algosection]{algorithm2e}
\usepackage{cleveref}

\usepackage{orcidlink}

\usetikzlibrary{quantikz}

\newcommand{\poly}{\mathsf{poly}}

\usepackage{amsthm}
\newtheorem{theorem}{Theorem}

\newtheorem{lemma}[theorem]{Lemma}
\newtheorem{corollary}[theorem]{Corollary}
\newtheorem{definition}[theorem]{Definition}

\newtheorem*{theorem*}{Theorem}
\numberwithin{theorem}{section}
\numberwithin{conjecture}{section}
\numberwithin{problem}{section}

\newmdtheoremenv[backgroundcolor=gray!10,
                 linewidth=0pt,
                 innerleftmargin=16pt,
                 innerrightmargin=16pt,
                 innertopmargin=6pt,
                 innerbottommargin=6pt,
            splitbottomskip=4pt]{protocol}[prot]{Game}

\newcommand{\negl}{\mathsf{negl}}

\newif\ifnotes
\notestrue

\title{Proofs of Ownership for Machine Learning Models}
\author{
{\Large Ran Canetti}\\
Boston University
\and
{\Large Shafi Goldwasser}\\
MIT and UC Berkeley
\and
{\Large Or Zamir}\\
Tel Aviv University
}

\begin{document}
%
\maketitle

\begin{abstract}
With the increasing adoption of Machine Learning, protecting model ownership has become an essential challenge.
We initiate a formal study of \emph{Proof of Ownership} for machine learning models: under what conditions can one prove that a stolen model originated from a particular creator?
We model proofs of ownership as a game among three parties: a  model owner, a  thief, and a judge.
The owner transforms the original model into a slightly perturbed model together with a proof of ownership.
The thief then obtains the transformed model and attempts to minimally modify it so that it remains useful but escapes detection as owned by the model owner.
Finally, the judge receives a model and a  proof of ownership, and must decide whether the given model is a modified version of some model created by the  model owner, or else the given model was developed independently.

Our main result is a dichotomy for classifiers in the black-box setting: Under standard cryptographic assumptions, ownership of models for some concept class can be proven in the above sense  {\em if and only if}   the concept class  is  not self-correctable, in a sense close to that of Blum, Luby and Rubinfeld, STOC'90.  The result is constructive and extends, with some variations,  to a number of related settings.  

\end{abstract}

\newpage
\tableofcontents
\newpage

\section{Introduction}

Machine learning models are increasingly valuable intellectual assets, often trained at significant cost.
This raises a fundamental question: Can the rightful  owner of a model prove ownership, in case that the model is  either leaked,  or simply stolen by a party that was given the right to use the model?  
We call this task \emph{Proving Model Ownership}.

If the owner has registered the model  ahead of time  in a registry that is universally trusted, and the leaked model remains unchanged, then asserting ownership is straightforward.
However, what if the model has been purposefully altered so as to avoid detection? Furthermore, what if no universally trusted registry exists, or else the model owner simply didn't use it? 
Can the model's owner still prove ownership, just by virtue of having developed the model herself? If so, then at what cost? 

In other words, we ask whether it is possible to transform a model into a form that simultaneously preserves the original predictive power and embeds in it an ``unremovable proof of ownership'': any process that modifies the leaked model so as to remove the proof of ownership necessarily renders the modified model useless.

Clearly, not all   models can (or should) admit meaningful proofs of ownership. Indeed ``proofs of ownership'' of  models that perform  a ``trivial task'', i.e. a task that can be easily learned from scratch,   are arguably meaningless.  But where should the line be drawn between those models that  can be ``owned'' and those that cannot?  Furthermore, where {\em can} we draw that line?

To reason about this question, we model proofs of ownership as a game between three participants: an \emph{owner}, a \emph{thief}, and a \emph{judge}.
The owner transforms a model~$M$ into a similar model~$M'$ together with a proof~$w$.
The thief obtains~$M'$ and attempts to modify it into~$M''$, aiming to remove any evidence of its origin while keeping it functionally effective.
Finally, the judge receives~$M''$ and~$w$ (which potentially includes also a purported "original" model $M$), and must decide whether~$M''$ is owned by the party that produced $w$, or else was developed independently.

A bit more concretely, a \emph{Proof of Ownership} for a class $\mathcal M$ of models consists of a marking algorithm $\mathcal{W}$ that, given a model $M$ generates a transformed model $M'$ along with a proof of ownership $w$, and a judge algorithm $\mathcal{T}$ that, given a model $M''$ and a claimed proof $w$, either accepts or rejects.  The following three properties should be satisfied:

(1)~\emph{Similarity}: the transformed model~$M'$ has essentially the same predictive power as $M$.\footnote{In principle, the predictive power of both $M$ and $M'$ would be measured against an underlying concept class.  However a simpler and essentially equivalent formulation only requires that $M'$ is a good predictor of $M$.}    

(2)~\emph{Unremovability}:  Let $M\in\mathcal M$, and let $(M',w)\leftarrow \mathcal W(M)$. It is infeasible to generate,  given $M'$, a model $M''$ that is even a moderate predictor of  $M'$ and such that $\mathcal T(M'',w)$ rejects. 

(3)~\emph{Soundness}: 
It is infeasible to generate, given a model $M\in \mathcal M$, a proof $w$ such that $\mathcal T(M,w)$ accepts.

There is still a lot to unpack before the above sketch can be made into a workable definition.  However, already at this point it is clear that the definition is unattainable as is:  the marking algorithm $\cal W$  directly breaks soundness.  To circumvent  this direct contradiction, we introduce  a {\em public parameters generation algorithm, $\cal G$. } Initially,  $\cal G$ samples public parameters $p$, which are then used both by $\cal W$  when generating $M'$  and by $\cal T$ when adjudicating ownership. Now, soundness holds with respect to models that were created regardless of $p$, whereas the marking algorithm uses $p$ to generate $M'$.

So, how to  effectively mark a model? Superficially, one might try to demonstrate ownership by adding a visible watermark to its ``code'' or weights, such as encoding a label saying ``Created by ClosedML''.
Such labeling trivially achieves similarity and soundness, but not unremovability: a thief can easily remove such a  mark without affecting functionality.
In fact, a fundamental observation by Barak et al.~\cite{BarakGoldreichImpagliazzoRudichSahaiVadhanYang2001} establishes that embedding watermarks directly in a program’s description provides no meaningful protection, as programs can be obfuscated.  In  other words, the judge algorithm $\mathcal T$ only meaningfully depend on  the {\em functionality} of the tested program $M"$;  in fact,  without loss of generality the judge algorithm   uses only oracle access to $M"$. We indeed follow this convention and thus focus on black-box protocols that access the models only via queries.

Another natural attempt is to embed some ``incriminating answers'' in the model, say, making it respond positively to ``Are you a model created by ClosedML?''.
However, the thief may easily override the model's answers to such a question and make it respond negatively instead.
Thus our central technical question becomes: \emph{under what conditions can proofs of ownership be made so that they change the underlying model only very little, and at the same time remain  persistent against any adversarial modification that stops short of rendering the model useless?}

Our first main conclusion is that the above question depends only on the concept, or function, being learned---and not on the training algorithm or resulting model.
That is, models  that learn certain types of concepts can be ownership-protected, while models that learn other types of concepts can inherently be stolen.

We next connect this question to a classical notion in complexity and learning theory:  \emph{self-correctability} \cite{BlumLR93}.
Recall that a function $f$ is self-correctable if there exists an algorithm $\mathcal{C}^{f'}$ that, given oracle access to any $f'$ similar to $f$, recovers $f(x)$ on all inputs with high probability.
Intuitively, a self-correctable function is one whose correct outputs can be efficiently reconstructed even from a corrupted implementation.
A self-correctable function cannot carry a persistent functional watermark, since a thief can use the self-correction process to reconstruct a clean canonical version and erase any embedded signal.
This observation was used by Goldwasser et al.~\cite{goldwasser2025oblivious} to remove possible backdoors from certain function families such as low-degree polynomials or functions with a bounded number of non-zero few Fourier coefficients; those are functions that are either easily learnable, or possess algebraic self-reducibility.

Our result shows that the above are not isolated examples but manifestations of a single principle:
\emph{self-correctability is the only fundamental obstruction to persistent proofs of model ownership.}

Specifically, we first formulate a natural relaxation of the original notion, whereby the self-corrected function $\bar f$ computed by $\mathcal{C}^{f'}$ can lightly differ from the original $f$, as long as it does not depend on the specific noisy version $f'$ (that is, ~$\bar f$ is some canonical version chosen identically for all noisy versions~$f'$ of~$f$).
We then establish a dichotomy for binary classifiers: informally, every model either admits a secure proof of ownership or else is close to a self-correctable function.
In other words, the ability to prove ownership is determined entirely by the mathematical structure of the underlying function, rather than by the specifics of the learning algorithm or model architecture.
Moreover, the connection is a constructive reduction: from any successful ``jailbreaking'' thief that removes ownership information, one can algorithmically construct a self-corrector for a function close to the original model; and vice versa, a self-corrector implies a generic removal strategy.
The resources needed by both algorithms are comparable.

This equivalence provides a precise conceptual boundary for protecting classification algorithms.
We also discuss extensions of this equivalence to the cases of  multi-class classifiers (for any constant number of labels) and  generative models.

To keep the treatment clear and focused, we restrict ourselves to the case where the thief (i.e. the adversary who is trying to  break  unremovability)  only has black-box access to the model $M'$.  Similarly,  the perturbed program  $M''$ can access the underlying model $M'$  via oracle access. More precisely,   $M''$ is formalized as a boolean circuit with the addition of any number of  `$M'$ gates', namely  gates that take input $x$ and output $M'(x)$. 

This restriction isolates the learning and complexity-theoretic aspects of proofs of ownership from the separate question of how to realize (or approximate) black-box access in white-box settings. Indeed, sufficiently strong program obfuscation (see e.g. ~\cite{Barak01, BCP13,cohen2016watermarking}) could  effectively turn a black-box protocol into a white-box one.  We leave the study of appropriate obfuscation mechanisms out of scope for this work.

Our definition of self-correction captures  model extraction, re-learning, and knowledge distillation~\cite{hinton2015distilling,tramer2016stealing,liu2025model} as special cases.
Indeed,  the ability to train an equivalent model (e.g., by distillation or by re-learning from data or queried labels---a task that inherently becomes easier given query access to our model)  implies self-correction:  the algorithm that ignores its oracle and simply queries the reconstructed model is a valid self-corrector in our sense.
 Black-box watermark removal strategies such as distillation are also captured by our formalism  as special cases of self-correction: they produce a model that replicates the target function without relying on the original oracle at inference time.

In light of this discussion,  our dichotomy theorem can be viewed as an overarching structural dichotomy that explains when  mitigation or distillation attacks can succeed against functional watermarking. In particular, our general characterization  can be used to provide a uniform treatment to  known attacks such as algebraic self-corrections or distillation.

We also note that the variant of self-correction captures common attack modalities: relearning from data, knowledge distillation, fine-tuning, algebraic self-reductions, and any wrapper that reconstructs or corrects the target function.
The dichotomy therefore explains, in one statement, when such procedures can erase ownership information without degrading utility and when, on the contrary, ownership can be made persistent.
Conceptually, our results place on firm footing a widely repeated intuition about using backdoors as watermarks~\cite{adi2018turning}.
We prove this intuition formally under natural black-box settings and show that self-correctability is the only obstacle to their formalization; we provide constructive reductions in both directions, which we view as a proof of concept for a broader theory of model ownership.
Given the broad applicability of the model and the characterization, we view the simplicity of the resulting boundary and reductions as an advantage: it cleanly isolates the sole obstruction and turns a diverse set of attacks into instances of a single principle.

\subsection{Our techniques}

While we refrain from using heavy cryptographic tools such as general-purpose program obfuscation (and instead opt for a more restricted model where adversaries have only oracle access to the marked model), our proof of ownership  does make essential and non-trivial use of cryptographic machinery.

Specifically, our proof of ownership makes crucial use of correlation intractable  hash functions:  The public parameter is a function  ${\cal H}_p$, drawn from a family that is correlation intractable against evasive relations that are computable in polynomial time, as in \cite{CanettiGH04,CCHLRRW19,PeikertS19,BrakerskiKM20}.  To mark a model $M$, the marking algorithm  chooses a random seed $w$ for a PRG $G:\{0,1\}^\lambda\rightarrow\{0,1\}^{\lambda d}$, where $\lambda$ is the security parameter and $d$ is the dimension of the model. $G(w)$ is then interpreted as $\lambda$ input points $x_1\ldots x_\lambda$. The witness of ownership is then $w$, and the  marked model $M'$ is  identical to $M$, except that $M'(x_i)={\cal H}_p(w)_i$.
Given oracle access to a suspect model $M''$, as well as  the public parameters $p$ and the witness $w$, the  testing algorithm  checks whether $M''$ agrees with ${\cal H}_p$ on significantly more than half of the input points defined by $w$.

Observe  that this scheme modifies $M$ only at a polynomial number of points.  Furthermore,  the model owner is not required to ever expose her secret model: the witness of ownership is only the random seed. Furthermore, the public parameters can be made completely transparent with no secret randomness. 

The crux of our proof is to show that this simple scheme is in fact universal: It is sound for any model that allows for proofs of ownership in the first place.

\subsection{Connection to Previous Works}

Two strands of prior work are most relevant: cryptographic \emph{function watermarking} and empirical \emph{neural network watermarking}. The cryptographic line studies unremovability and unforgeability for precise, algorithmic functionalities (not ML models), with realizations for PRFs and public-key primitives. The empirical line proposes practical methods to tag trained networks or their outputs, often via backdoors or behavioral signatures, together with attack papers that remove or detect such tags. Our work differs in goal and scope: we formalize \emph{model ownership} as a game for black-box models and prove a dichotomy that ties the feasibility of persistent ownership to \emph{self-correctability} of the underlying function.

\paragraph{Watermarking of cryptographic functions.}
Following early  impossibility results in Barak et al. \cite{BarakGoldreichImpagliazzoRudichSahaiVadhanYang2001},  Cohen, Holmgren, Nishimaki, Vaikuntanathan, and Wichs rigorously defined  watermarking for  cryptographic \emph{capabilities}, and devised  iO-based schemes for pseudorandom functions and related primitives, with strong unremovability guarantees \cite{cohen2016watermarking}. Subsequent works pursued analogous goals for public-key functionalities and implementations (e.g., encryption and signatures). A key advance replaced iO with standard assumptions for certain targets: Kim and Wu constructed watermarking for cryptographic functionalities from lattice assumptions, giving unremovability against arbitrary removal strategies under LWE-type hardness \cite{kim2017watermarking}. 
These works watermark \emph{pseudorandom} or otherwise unstructured objects whose input-output behavior, without extra information, is indistinguishable from random. Our setting instead concerns predictive models and ``useful functions" instead. 

\paragraph{Watermarking of ML models.}
Within machine learning, a broad empirical literature proposes to watermark trained networks. Uchida et al. embed marks into weights via regularization and verify ownership by reading the embedded bits; this primarily targets white-box settings \cite{uchida2017embedding}. To enable black-box verification, a common approach is to use \emph{backdoor-style} behavioral keys: Adi et al. introduce backdoor watermarks whose presence can be tested via special queries, providing both a modeling framework and extensive experiments \cite{adi2018turning}. A large additional body of such watermarks and empirical attacks on them exists. 
Our formalization differs in two ways: first, we separate \emph{soundness} from \emph{unremovability} in a computational framework; second, we prove that the sole barrier to a \emph{provably} unremovable ownership (in the black-box model) is the existence of a self-corrector for a function close to the model. This yields a structural dichotomy rather than an empirical assessment. Nonetheless, the idea of using backdoors as a form of model watermarking is used in our construction as well.

\paragraph{Backdoors and their relation to watermarking.}
Backdoors in ML instantiate targeted behaviors that are dormant on typical inputs and activate on secret triggers.  Empirically, several works propose to \emph{use} such triggers as black-box watermarks for ownership verification, notably backdoor-based schemes of Adi et al.\ \cite{adi2018turning}.  On the theoretical side, Goldwasser et al. show that undetectable backdoors can be planted so that a backdoored model is computationally indistinguishable from a clean one \cite{GoldwasserKVZ22}, and follow-up work studies \emph{oblivious} backdoor removal without detection \cite{goldwasser2025oblivious}. 
Our framework explains these phenomena function-theoretically: backdoor-as-watermark approaches persist only when the underlying function (or any function close to it) is not self-correctable; otherwise a remover can reconstruct clean behavior while preserving overall performance.

\paragraph{Watermarking \emph{outputs} vs.\ model  ownership.}
A separate line of work  watermarks \emph{individual outputs} of generative models rather than the overall functionality.  For text, token-selection and statistical coding schemes embed detectable patterns in LLM generations \cite{kirchenbauer2023watermark}, with follow-up theory on undetectable or low-detectability designs \cite{ChristGZ24}.  For images, diffusion-model watermarking embeds signals into latent or pixel domains and is reviewed in recent work \cite{zhao2023recipe}.  These approaches certify the ownership of a \emph{particular artifact} and can remain meaningful even when the underlying model changes.  By contrast, our focus is on \emph{functional watermarking} of the model itself: the verification inspects the mapping \(M: \{0,1\}^d\to\{0,1\}\) rather than a produced sample.  The two problems are therefore not equivalent.  Output watermarking addresses content attribution and distributional detection, while model ownership aims to embed information in the overall functionality of a model so that ownership persists through any modification that preserves any meaningful task quality. 

\paragraph{Formal mitigations and other works on watermarking criteria.}
Recent formal frameworks analyze mitigation and watermarking from a cryptographic viewpoint.  
Goldwasser, Shafer, Vafa, and Vaikuntanathan give \emph{oblivious} backdoor removal without detection by exploiting random self-reducibility of structured label families, showing that mitigation can succeed even when detection provably fails \cite{goldwasser2025oblivious}.  
Gluch, Turan, Nagarajan, and Pokutta study a game-theoretic landscape for discriminative tasks and prove that at least one of three mechanisms must exist: a watermark, an adversarial defense, or a transferable attack; they further connect transferability to cryptographic hardness \cite{turan2024good}.  In their setting, the watermark attack is tailored to a known-in-advance input and the players operate under explicit time budgets.  
Gluch and Goldwasser formalize defense by detection versus defense by mitigation and establish equivalences for classification as well as separations for generative tasks under cryptographic assumptions \cite{gluch2025cryptographic}.  
Our work is complementary: we study functional \emph{ownership} and obtain a dichotomy specific to ownership proofs in the black-box model, where persistent ownership exists if and only if the underlying function is not close to a self-correctable function.  In particular, our condition abstracts over concrete defenses and captures re-learning, distillation, and related mitigation strategies as instances of self-correction; moreover, our model-stealing notion aligns with standard extraction formulations used in learning and security.

\paragraph{Organization.}
Section \ref{sec-defs} defines  proofs of ownership and self correctors.  Section \ref{sec-results} then presents and discusses our main results.  Section \ref{sec:ind}  presents our proof  of ownership, its analysis,  and a number of extensions.  Section \ref{sec:discuss} contains additional discussion and proposes subsequent research directions.

\section{Proofs of Ownership and Self Correctors}
\label{sec-defs}

We formalize the definitions of proofs of ownership and correctable canonicalizations (which are our extensions of self-correctors \cite{BlumLR93}). 

Throughout, $\lambda \in \mathbb{N}$ denotes the security parameter, $\negl(\cdot)$ denotes a negligible function, and all algorithms are probabilistic polynomial time unless stated otherwise. 

We are primarily interested in families  $\mathcal M$ of models where each model $M\in\cal M$ is a circuit from $\{0,1\}^d$ to $\{0,1\}$. 

More precisely, we will be interested in ensembles of model families, with a family for each value of the security parameter $\lambda$, and where the size of circuits in the $\lambda$-th family is polynomial in $\lambda$. Still, the presentation will often suppress the dependence on $\lambda$.

\paragraph{Similarity relations.}
Notions of distance  (or, similarity) between models  are key components of our treatment.  Most of the forthcoming definitions are meaningful for multiple notions of similarity and are thus parameterized by a generic similarity relation `$\approx$' between models. 

Still, our algorithmic results pertain primarily to  the  two structural classes of 
similarity relations.  We thus define them here for sake of concreteness.

\begin{itemize}
\item \textbf{Distributional approximation}, denoted $\approx_{\varepsilon}^{\mathcal{D}}$.  
    Let $\mathcal{D}$ be a distribution over $\{0,1\}^d$ and let $\varepsilon>0$.  
    Let $M,M':\{0,1\}^d\to\{0,1\}$. We write $M'\approx_{\varepsilon}^{\mathcal{D}} M$ if
    \[
\Pr_{x\sim\mathcal{D}}\big[M'(x)=M(x)\big]\ge 1-\varepsilon.
    \]
    This is the standard learning-theoretic notion of approximation, corresponding to binary loss and distribution $D$. 
    
    \item \textbf{Evasive difference}, denoted $\approx^{\mathcal E}$.  This significantly tighter measure of similarity follows common notions of indistinguishability from cryptography.
    Intuitively, it means that no polynomial time algorithm can locate any disagreement between the two functions, given oracle access. Furthermore, even if that distinguisher is given a point -- possibly one in which the functions \emph{do} differ -- it still can't find any additional point of disagreement.
    Formally, we measure the (asymptotic) distance between a single fixed model and a distribution over models: 
    Let $M:\{0,1\}^d\to\{0,1\}$  and let $\mathcal{M}'$ be a distribution over functions $M':\{0,1\}^d\to\{0,1\}$.
    We say that~$\mathcal{M}' \approx^{\mathcal E} M$ if for every PPT adversary $\mathcal{A}$ and every point~$z\in \{0,1\}^d$ it holds that
    \[
    \Pr\big[\,M'\sim\mathcal {M'},\ x\leftarrow\mathcal{A}^{M',M}(1^\lambda,z):M(x)\neq M'(x)\;\wedge\;x\neq z\,\big]\le \negl(\lambda).
    \] 
    We later abuse notation and use~$M'$ to denote either the distribution or a specific sample from it, which is clear from context.
\end{itemize}

\subsection{Proofs of ownership}

We start by defining the ``mechanics'' of  proofs of ownership. Next we define out three main requirements:  quality preservation, soundness, and unremovability. 

\begin{definition}[Proofs of ownership  for black-box models]\label{def:mws}
A \emph{Proof of ownership} for a class~$\mathcal{M}$ of models consists of a setup algorithm $\mathcal G$, a marking algorithm $\mathcal W$, and a testing algorithm $\mathcal T$.  The setup algorithm samples public parameters $p\leftarrow \mathcal G(1^\lambda)$.  Given $p$ and a description of a classifier $M\in \mathcal{M}$, the marking algorithm $\mathcal W$ outputs a pair $(M',w)$, where $M'$ is the marked model and $w$ is an ownership witness.  
The testing algorithm $\mathcal T$ takes inputs $(p,w)$, obtains oracle access to a model $M''$  and outputs $\mathsf{accept}$ or $\mathsf{reject}$.  

\end{definition}

Here $M''$ is a \emph{suspect} model that may have been adversarially generated  with access to $M'$.  Importantly, the only information that the testing algorithm $\mathcal T$ obtains from the claimed owner  is the  ownership witness $w$. In particular, $\mathcal T$ does not obtain access to either $M$ or $M'$ (they can be encoded as part of~$w$, but the testing algorithm cannot verify their authenticity).

\paragraph{Soundness.}
The first security  requirement is soundness: it should be infeasible to come up with a valid ownership witness for a model $M$ that was generated independently of the public parameters. Furthermore, this property should hold even given the public parameters and the full description of $M$:

\begin{definition}[Soundness]\label{def:soundness-unreg}
A Proof of Ownership $(\mathcal G,\mathcal{W},\mathcal{T})$ is \emph{sound} if for every model $M\in \mathcal{M}$ and every adversary $\mathcal{A}$,
\[
\Pr\big[\mathcal G(1^\lambda)\rightarrow p, \ \mathcal A(p,M)\rightarrow w: \mathcal{T}^M(p,w)=\mathsf{accept}\big]\le \negl(\lambda).
\]
\end{definition}

\paragraph{Quality preservation.}
Quality preservation requires that the marked model should preserve the functional behavior of the original model, namely its predictive quality.

\begin{definition}[Quality preservation]\label{def:equiv}
A Proof of Ownership  $(\mathcal G,\mathcal W,\mathcal T)$ satisfies \emph{Quality Preservation} with respect to a relation `$\approx$' if for every model $M\in\mathcal{M}$,
\[
\Pr\big[\,p\leftarrow \mathcal G(1^\lambda),\ (M',w)\leftarrow \mathcal W(p,M):\ M'\approx M\,\big]\ge 1-\negl(\lambda).
\]
\end{definition}
We note that this property can be naturally parameterized by parameterizing the similarity relation $\approx$.

\paragraph{Unremovability.}
We turn to defining  unremovability. 
This property captures robustness to attempts to remove the ownership information.  Any model generated from the marked model $M'$ alone must either cease to be functionally similar to $M'$ or else continue to validate the witness.

We concentrate on  adversaries that have only black box access to the marked model.  
More specifically, we consider adversaries 
that operate in two stages:  In the first stage, the adversary is given free black box access to a watermarked model $M'$,  and is tasked to output a description of a model $M''$, where $M''$ is a circuit that includes, in addition to standard computational gates, also ``$M'$ gates''. This should be viewed as the thief being allowed to use~$M'$ when training~$M''$, and also to ``wrap'' it during inference-time.  
In the second stage,  the model $M''^{M'}$ is inspected for  quality and ownership.

\begin{definition}[Unremovability]\label{def:persistence}
A Proof of Ownership $(\mathcal G,\mathcal W,\mathcal T)$ is \emph{unremovable} with respect to similarity relation `$\approx$' if for any efficient adversary $\mathcal R$, and for all but a negligible fraction of the models~$M\in \mathcal{M}$ we have:
\[
\Pr\big[\,p\leftarrow \mathcal G(1^\lambda),\ (M',w)\leftarrow \mathcal W(p,M),\ M''\leftarrow \mathcal R^{M'}(p):\ \mathcal T^{M''^{M'}}(p,w)=\mathsf{reject}\ \wedge\ M''^{M'}\approx M'\,\big]\le \negl(\lambda).
\]

\end{definition}

\paragraph{Discussion.}
It may be instructive to note the following two points regarding proofs of ownership:
First, while it might appear tempting to strengthen the unremovability requirement to hold for {\em any} $M\in\mathcal M$, such  stronger requirement is ruled out by the universal quantifier over adversaries.  Specifically,  any model $M$ can be paired with the  adversary that  outputs the fixed model $M$ (regardless of the computational model).

Second, the similarity relation used in the quality preservation requirement and the one used in (either version of) the unremovability requirement are in clear tension: We would like the first to be as small as possible, while having the second be as large as possible.  This tension translates directly to the realm of self correctors,  and is the basis for the correspondence demonstrated in this work.

A  {\bf secure proof of ownership}  is a proof of ownership that is sound,  quality preserving, and unremovable.

\subsection{Self-correctors}

We define self-correctors with respect to  similarity relation `$\approx$' and a function class~$\mathcal{F}$.
As is common in definitions of self-correction, we attempt to correct a corrupted function into the closest member of a simple and known function family~$\mathcal{F}$ (e.g., linear functions, low-degree polynomials, etc.).

\begin{definition}[self-corrector \cite{BlumLR93}]\label{def:corrector}
Let $ \mathcal{F} \subseteq \left( \{0,1\}^d\to\{0,1\}\right) $ and let `$\approx$' be a similarity relation on models.  
A randomized oracle algorithm $C^{(\cdot)}$ is a \emph{self-corrector for $\mathcal{F}$ with respect to `$\approx$'} if for every~$f\in  \mathcal{F}$, every function $f'$ that satisfies $f'\approx f$, and for every $x\in  \{0,1\}^d$,
\[
\Pr\big[\,C^{f'}(x)=f(x)\,\big]\ge 1-\negl(\lambda).
\]
\end{definition}

The standard definition of self-correction, captured above, is only meaningful when every two functions in the class~$\mathcal{F}$ are efficiently distinguishable from each other, given only   oracle access. 
To cover more general function families, we extend the definition of self-correction by allowing the corrector  to pick a ``canonical function'' ~$\bar{f}$ that is close to the original $f$, and then fully correct any noisy version~$f'$ of~$f$ into~$\bar{f}$ (rather than to~$f$).   Crucially, $\bar{f}$  need not be in ~$\mathcal{F}$:
for instance, if~$\mathcal{F}$ is the family of all functions~$f$ that are functionally very close to a constant,  ~$\bar{f}$  may well be a constant function.
Note that~$\bar{f}$ must be a function only of the self corrector and~$f$;  in particular,  it must not depend on  the difference~$f'\Delta f$.  

This extended notion, which we call  {\em correctable canonicalization,} naturally distinguishes between two different similarity relations: The first is the similarity between the canonical $\bar f$ implied by the self corrector and the original $f$. The second is the minimal similarity level of the given $f'$ that is correctable at all. In the definition below we denote the first similarity (the one that we want to maximize) by $\approx^\uparrow$,  while the second similarity (the one we want to minimize) is denoted by $\approx^\downarrow$.

\begin{definition}[correctable canonicalization]\label{def:canon}
Let $ \mathcal{F} \subseteq \left( \{0,1\}^d\to\{0,1\}\right) $ and let `$\approx^\uparrow$' and `$\approx^\downarrow$' be similarity relations on models.  
A randomized oracle algorithm $C^{(\cdot)}$ is correctable canonicalization for $\mathcal{F}$ if for every~$f\in  \mathcal{F}$ there exists~$\bar{f}\approx^\uparrow f$, such that for every function $f'$ that satisfies $f'\approx^\downarrow f$, and for every $x\in  \{0,1\}^d$,
\[
\Pr\big[\,C^{f'}(x)=\bar{f}(x)\,\big]\ge 1-\negl(\lambda).
\]
\end{definition}

\section{Result Statements: Proofs of ownership vs. Self Correctors}
\label{sec-results}

We are finally ready to state our results concretely.
For clarity of exposition, we focus on binary classifiers~$M: \{0,1\}^d\rightarrow\{0,1\}$. As discussed in Section~\ref{subsec:multiclass}, the same statements extend to classifiers with any constant number of labels via a simple reduction to the binary case.

We study the black-box setting.
In a nutshell,   we  first provide a characterization of when proofs of ownership are possible. Furthermore,  we construct a general  proof of ownership for any model family $\mathcal M$,  and demonstrate that it is {\em universally optimal}: if a secure proof of ownership for $\mathcal M$ exists at all, then our universal construction is secure.

More concretely,  the proof is always sound and quality preserving with respect to $\approx^{\cal U}_{\negl(\lambda)}$. Next, for any family $\mathcal M$ of models we have:  (a)  as long as there do not exist correctable canonicalizers for $\mathcal M$ and $\approx_\downarrow$, the proof is unremovable with respect to a similarity notion related to $\approx_\downarrow$. (b) if  correctable canonicalizers for $\mathcal M$ and $\approx_\downarrow$  do exist,  then no  proofs of ownership for $\mathcal M$ (and similarity notions related to $\approx_\uparrow$ and $\approx_\downarrow$) exist.

We begin with demonstrating the easier direction, which requires almost no additional assumptions beyond the definitions above.  If a model computes a function that is close to a self-correctable function then it does not admit an unremovable proof of ownership.
\begin{theorem}\label{thm:easy-dir}
    Let~`$\approx^\uparrow$',`$\approx^\downarrow$' be relations.
    Suppose there exists a correctable canonicalization~$C^{(\cdot)}$ for a class~$\mathcal{M}$ with respect to~(`$\approx^\uparrow$',`$\approx^\downarrow$').
    Then, every proof of ownership $(\mathcal G,\mathcal W,\mathcal T)$ for~$\mathcal{M}$ that satisfies Similarity and Soundness with respect to~`$\approx^\downarrow$' can be removed with respect to~`$\approx^\uparrow$'.
    Moreover, the resulting removal algorithm uses the same time, oracle complexity, and auxiliary resources as~$C^{(\cdot)}$.
\end{theorem}
\begin{proof}
As the proof is rather short, we provide it here in full.
Let $(\mathcal G,\mathcal W,\mathcal T)$ be any proof of ownership for~$\mathcal{M}$.

We construct a remover $\mathcal R$ that, given oracle access to $M'$ that was constructed by~$W(p,M)$ for some~$M\in \mathcal{M}$ and $p\leftarrow \mathcal G(1^\lambda)$, outputs the model
\[
M''(x) := C^{M'}(x).
\]
By Similarity, $M'\approx^\downarrow M$ except with negligible probability.  
The definition of a correctable canonicalization therefore implies that, for every input $x$, $M''(x)=\bar{M}(x)$ except with negligible probability.  Furthermore, $M''\approx^\uparrow M.$

It remains to show that the ownership witness $w$ does not verify against~$M''$.  Suppose toward contradiction that
\[
\Pr\big[\,\mathcal T^{M''}(p,w)=\mathsf{accept}\,\big]
\] 
is non-negligible.  Since $M''$ and~$\bar{M}$ have the same functionality on every input up to negligible error and $\mathcal T$ is an efficient black-box test, the same witness~$w$ would then make $\mathcal T^{\bar{M}}(p,w)$ accept with non-negligible probability.  But this contradicts Soundness: the target model~$\bar{M}$ is fixed independently of the public setup~$p$.  Therefore,
\[
\Pr\big[\,\mathcal T^{\bar{M}}(p,w)=\mathsf{accept}\,\big]\le \negl(\lambda).
\]
Hence $\mathcal R$ outputs a model $M''\approx^\uparrow M$ but causes the verification to reject, breaking unremovability.  The running time and oracle complexity of $\mathcal R$ are exactly those of the corrector $C^{(\cdot)}$.

\end{proof}

We remark again on a subtlety in the above statement: it suffices to have a self-corrector not necessarily for the model~$M$ itself, but for any canonicalization~$\bar{M}$ that is close to it.  
This relaxation is unavoidable, since resistance to watermarking extends not only to self-correctable functions but also to functions that are sufficiently close to them.  
For example, a point function may conceal the location of its unique nonzero point, yet its approximation by the identically zero function, which is indistinguishable from it, retains the same utility while eliminating any watermark information.

The more difficult direction is thus the other one: demonstrating  that self-correction is the {\em only } obstacle to ownership. To this end, we construct a proof of ownership such that any adversary who can break its unremovability can also construct a correctable canonicalization for the same class of models. 
As in the easier direction, our proof is a reduction that uses such a successful adversary to explicitly construct the self-corrector, while using equivalent resources:  

\begin{theorem}\label{thm:inds}
Assume there exist correlation intractable function families with respect to polytime computable evasive relations. Then for every class of models~$\mathcal{M}$ there exists a proof of ownership $(\mathcal G,\mathcal W,\mathcal T)$ such that the following holds. Every efficient adversary $\mathcal{R}$ that breaks the protocol's unremovability with respect to the evasive difference relation~$ \approx^{\mathcal E}$, can be efficiently converted into a correctable canonicalization $C^{(\cdot)}$ for~$\mathcal{M}$ with respect to ($\approx^\mathcal{U}_{1/\poly(\lambda)}$,$ \approx^{\mathcal E}$).  
The reduction from $\mathcal{R}$ to $C^{(\cdot)}$ runs in comparable time, query complexity, and all other resources and assumptions.
\end{theorem}

\paragraph{Beyond evasive difference.}
The reduction above is stated under the evasive difference relation $ \approx^{\mathcal E}$, which is
appropriate when both the owner and the thief are restricted to making only negligible,
computationally hidden changes.
In many realistic scenarios, however, it is only reasonable to demand that the \emph{owner} preserves
evasive difference (so releasing $M'$ does not visibly degrade the model), while allowing a
\emph{thief} to introduce a noticeable error on an $\varepsilon$ fraction of inputs.
We formalize this by requiring $M'  \approx^{\mathcal E} M$ but only asking that the thief outputs a model
$M''$ that is distributionally close to $M'$ under the uniform distribution,
namely $M'' \approx^{\mathcal{U}}_{\varepsilon} M'$.
Our main construction continues to yield a meaningful implication: a successful watermark remover in
this regime implies \emph{distributional} self-correctability for a function close to the original
model.

\begin{theorem}[Distributional self-correctability]
\label{thm:dist-correct}
Assume there exist correlation intractable function families with respect to polytime computable evasive relations. Then 
for any~$\varepsilon \in (0,\frac{1}{3})$ and class of models~$\mathcal{M}$,
there exists a proof of ownership $(\mathcal G,\mathcal W,\mathcal T)$ such that, if $p\leftarrow \mathcal G(1^\lambda)$ and $(M',w)\leftarrow \mathcal W(p,M)$, then $M'  \approx^{\mathcal E} M$ and the following holds.
If an efficient adversary $\mathcal{R}$ breaks its unremovability with respect to distributional distance $M'' \approx_{\varepsilon}^{\mathcal{U}} M'$, then one can efficiently
construct a correctable canonicalization $C^{(\cdot)}$ with respect to ($\approx^\mathcal{U}_{3\varepsilon + 1/\poly(\lambda)}$,$ \approx^{\mathcal E}$).
The construction of $C^{(\cdot)}$ uses comparable time, oracle complexity, and resources to~$\mathcal{R}$.
\end{theorem}

In sum,  we show:

\begin{corollary}\label{cor:inds}
Assume there exist correlation intractable function families with respect to polytime computable evasive relations. Then 
for any class~$\mathcal{M}$ of binary classifiers,  
a proof of  ownership  $(\mathcal G,\mathcal W,\mathcal T)$ that is unremovable with respect to distributional distance~$\varepsilon>1/\poly({\lambda})$
exists \textbf{if and only if} there does not exist a correctable canonicalization $C^{(\cdot)}$ for~$\mathcal{M}$ with distributional distance~$\Theta(\varepsilon)$ between each model and its canonicalization.
That is, in the black-box setting, a model class admits a secure proof of ownership if and only if it is not close to a self-correctable class.
\end{corollary}

\section{ The Proof of Ownership }\label{sec:ind}

In this section we prove Theorem~\ref{thm:inds} by constructing a proof of ownership that is secure under the evasive difference relation~$ \approx^{\mathcal E}$.
For simplicity we assume the input space is the set of binary vectors~$ \mathcal{X}=\{0,1\}^d$,  where $d$ is polynomial in the security parameter $\lambda$, and that the models are deterministic --- which can be assumed as we can fix the randomness without degrading the accuracy of a classifier.

\subsection{High-Level Overview}

Our ownership algorithm~$\mathcal{W}$ transforms a given model~$M$ into a pair~$(M',w)$ as follows.  
First, it samples a random string~$w$ and uses it to define a polynomial-size  subset of the input space, denoted~$\mathcal{X}_w\subseteq\mathcal{X}$.  
The subset~$\mathcal{X}_w$ is uniquely determined by~$w$ 
but is computationally hard to predict without it.  
The transformed model~$M'$ behaves identically to~$M$ on all inputs outside of~$\mathcal{X}_w$, while on points inside~$\mathcal{X}_w$ it produces  distinguishable outputs that  are determined by the public parameter $p$  and serve as the watermark.  
The testing algorithm~$\mathcal{T}$, given~$(M',w)$, reconstructs~$\mathcal{X}_w$ from~$w$ and verifies that~$M'(\mathcal{X}_w)$ indeed exhibits correlation with the designated behavior on those inputs.

\paragraph{Similarity and Completeness.}
Similarity follows immediately: the model~$M'$ differs from~$M$ only on a polynomial-size, computationally evasive subset~$\mathcal{X}_w$, and therefore~$M' \approx^{\mathcal E} M$.  
Moreover, since the tester~$\mathcal{T}$ knows~$w$ and can reconstruct~$\mathcal{X}_w$, it correctly verifies the watermark.

\paragraph{Soundness.}
Soundness follows from the unpredictability of the  function $H_p$ determined by the public parameter $p$: given  a model $M$, it is infeasible to come up with a string $w$ such that $M(\mathcal{X}_w)$  has significant agreement  with $H_p$.

\paragraph{Unremovability.}
The core property of the construction is unremovability.  
Intuitively, the subset~$\mathcal{X}_w$ is both small and pseudorandom: it could contain any point in~$\mathcal{X}$.  Without knowledge of $w$, the only indication that a point $x$ might be in   $\mathcal{X}_w$ is that $M'(x)=H_p(x)$; however the construction will guarantee that the same will hold for a large fraction of the points in $\mathcal{X}$. We formalize this intuition by demonstrating that 
 any adversary~$\mathcal{R}$ attempting to remove the watermark  must fall into one of three cases:

\begin{enumerate}
    \item \textbf{(Similarity case)} The modified model~$M''$ behaves identically to~$M'$ on a large fraction of inputs.  
    In this case, it necessarily preserves the watermark behavior on~$\mathcal{X}_w$, and the tester~$\mathcal{T}$ will still accept.  
    Hence,~$\mathcal{R}$ fails.

    \item \textbf{(Deviation case)} The modified model~$M''$ changes the outputs on~$\mathcal{X}_w$ so that~$\mathcal{T}$ rejects.  
    Since~$\mathcal{R}$ cannot identify~$\mathcal{X}_w$, these modifications must affect a non-negligible fraction of other inputs as well, implying that~$M''$ differs significantly from $M'$ or~$M$.  
    Thus,~$\mathcal{R}$ again fails.

    \item \textbf{(Independence case)} The adversary avoids both previous cases by constructing~$M''$ that does not rely on queries to~$M'(x)$ when computing its own output on~$x$.  
    In this case,~$M''$ effectively reconstructs the outputs of~$M'$ without access to them, which is precisely the behavior of a \emph{self-corrector}.
\end{enumerate}

In all cases, either the adversary fails to remove the watermark or it implicitly constructs a self-corrector for~$M'$.  
This establishes the desired reduction from watermark removal to self-correction.

\subsection{The  Construction }\label{subsec:detailed}

Let $\lambda$ be the security parameter and let $d \gg \lambda$ denote the input dimension.  
We first describe a construction in  the random oracle model  (i.e.,  a construction that assumes that everyone has oracle  access to a random function $H: \{0,1\}^\lambda \rightarrow \{0,1\}^{\lambda}$). Next we describe how to instantiate  $H$ using a  correlation intractable hash function initialized using the public parameters chosen by the setup algorithm~$\mathcal{G}$.
In addition to the random function $H$, we also use a pseudorandom generator  $G: \{0,1\}^\lambda \rightarrow \{0,1\}^{d\lambda}$.

\paragraph{The public parameters generation algorithm  $\mathcal{G}$.}  Publish a random function $H: \{0,1\}^\lambda \rightarrow \{0,1\}^{\lambda}$.

\paragraph{The marking (prover) algorithm $\mathcal{W}$.}
Given a model $M:\{0,1\}^d\to\{0,1\}$, the marking algorithm proceeds as follows.

\begin{enumerate}
    \item Sample a seed $w \leftarrow \{0,1\}^{\lambda}$ for $G$.
    \item Let $G(w)=(x_1,\ldots,x_\lambda)$, where each $x_i\in\{0,1\}^d$, and let $H(w)=(y_1,\ldots,y_\lambda)$, where each $y_i\in\{0,1\}$.

    \item Define the watermarked model $M'$ by
    \[
        M'(x) :=
        \begin{cases}
            y_i, & \text{if } x = x_i \text{ for some } i\in[\lambda], \\
            M(x), & \text{otherwise.}
        \end{cases}
    \]
\end{enumerate}
The output of $\mathcal{W}$ is $(M',w)$.

\paragraph{The judge (tester) algorithm $\mathcal{T}$.}
Given $(M',w)$, the tester computes $x_1,\ldots,x_\lambda=G(w)$ and $y_1,\ldots,y_\lambda=H(w)$ and accepts if $M'(x_i)=y_i$ for at least~$\left(\frac{1}{2}\lambda + \lambda^{2/3}\right)$ of the indices $i\in[\lambda]$.

\begin{lemma}\label{lem:complete}
$({\cal G,W,T})$ is quality preserving.  Furthermore,  for any honestly generated pair $(M',w)\leftarrow \mathcal W(p,M)$, we have that $\mathcal T^{M'}(p,w)$ accepts  with probability~1.
\end{lemma}
\begin{proof}
By the pseudorandomness of $G$, finding even a single marked point $x_i$ without knowing $w$ is computationally infeasible.  
Hence, $M'$ and $M$ are computationally indistinguishable: $M'  \approx^{\mathcal E} M$.  
Moreover, since $M'(x_i)=y_i$ for all $i\in[\lambda]$, the tester $\mathcal{T}$ accepts $(M',w)$ with probability~$1$.  
Thus the scheme satisfies Similarity, and honest verification succeeds with probability~1.
\end{proof}

\begin{lemma}\label{lem:sound}
$({\cal G,W,T})$ is sound. 
\end{lemma}
\begin{proof}
Let $M$ be any fixed model generated independently of $H$.
Then for any $w\in\{0,1\}^\lambda$, the probability (over the choice of $H$) that $M(G(w)_i)=H(w)_i$ for at least~$\left(\frac{1}{2}\lambda + \lambda^{2/3}\right)$ of the indices $i\in[\lambda]$, a number that deviates from the expectation by~$\Theta(\lambda^{1/6})$ standard deviations, is $\negl(\lambda)$.  
Consequently, any adversary that makes at most $\tau$ queries to $H$ and then outputs a proof $w$ can cause $\mathcal T^{{M}}(p,w)$ to accept with probability at most $\tau\cdot  \negl(\lambda)$. 
\end{proof}

\begin{lemma}\label{lem:unrem}
If an adversary $\mathcal{R}$ breaks the unremovability of $(\mathcal G,\mathcal W,\mathcal T)$,  
then it can be used to construct a correctable canonicalization $C^{(\cdot)}$ for~$\mathcal{M}$ with respect to ($\approx^\mathcal{U}_{1/\poly(\lambda)}$,$ \approx^{\mathcal E}$)  
with comparable time and oracle complexity.
\end{lemma}
\begin{proof}
Let $(M',w)\leftarrow \mathcal{W}(M)$ and let $M''\leftarrow \mathcal{R}^{M'}(M')$ be the adversary’s modified model.  
In the black-box setting, $\mathcal{R}$ can only query $M'$ during its execution and can produce $M''$ as an oracle circuit  
$\mathcal{C}^{M'}$, so that $M''(x) = \mathcal{C}^{M'}(x)$.

Since $M' \approx^{\mathcal E} M$, we may assume that $\mathcal{R}$ never observes a point where $M$ and $M'$ differ during construction, or during evaluation besides the given input point.  
Therefore, for any input $x$, the value $\mathcal{C}^{M'}(x)$ depends only on $(M, x)$ and possibly on the single bit $M'(x)$.  
We define
\[
g_M(x,b) := \mathcal{C}^{M_b}(x),
\]
where $M_b$ is the model that behaves like $M'$ except that its output on $x$ is fixed to $b$.  
This allows us to simulate $g_M(x,b)$ for both $b\in\{0,1\}$ given oracle access to $M'$.
Note that~$g_M(x, \cdot)$ is merely a deterministic binary function from one bit to one bit.
Each input $x$ thus falls into one of three cases:
\begin{enumerate}
    \item $g_M(x,0)=g_M(x,1)$: $\mathcal{C}^{M'}(x)$ ignores $M'(x)$.
    \item $g_M(x,b)=b$: $\mathcal{C}^{M'}(x)$ reproduces $M'(x)$.
    \item $g_M(x,b)=1-b$: $\mathcal{C}^{M'}(x)$ inverts $M'(x)$.
\end{enumerate}

Let $D(x)\in\{1,2,3\}$ denote the type of $x$ as defined above; $D(x)$ depends only on $(g_M(x,0),g_M(x,1))$ and is therefore efficiently computable by $\mathcal{R}$ without knowledge of $w$.
We define the canonicalization~$\bar{M}=g_M(x,0)$ of~$M$, which as stated depends only on~$M$ itself and not on the negligible difference between~$M'$ and~$M$ with all but negligible probability.

Our main observation is that since $D$ is efficiently computable and independent of $w$, the distribution of $D(x)$ over $\mathcal{X}_w$ and over $\mathcal{X}$ differ only negligibly (otherwise $\mathcal{R}$ could distinguish the PRG output).  

Since $M''  \approx^{\mathcal E} M'$, the fraction of $x$ for which $D(x)=3$ is negligible, and thus also is that fraction within~$\mathcal{X}_w$.

On points of~$\mathcal{X}_w$ of type~$(1)$,~$M''$ identifies with~$\bar{M}$ that is defined independently of the difference between~$M$ and~$M'$ and thus independently of~$w$. Thus, with all but negligible probability the number of points with $M(G(w)_i)=H(w)_i$ of Type~$(1)$ will be of distance at most~$\frac{1}{2}\lambda^{2/3}$ from half of the points of this type.
If $\mathcal{R}$ successfully fools the judge, then, the fraction of $x\in \mathcal{X}_w$ for which $D(x) = 2$ must be at most~$3\lambda^{-1/3}$.
Hence, $D(x)=2$ only on a~$3\lambda^{-1/3}$ fraction of $\mathcal{X}$.
 
We conclude that for all but a $1/\poly(\lambda)$ fraction of input points, $D(x)=1$; that is, $\mathcal{C}^{M'}(x)$ ignores $M'(x)$ entirely and returns the independently defined~$\bar{M}(x)$.  
Therefore, the function
\[
\bar{M} := g_M(x,0)
\]
is efficiently computable by $\mathcal{R}$, satisfies $\bar{M} \approx^\mathcal{U}_{1/\poly(\lambda)} M''  \approx^{\mathcal E} M'$, and can be evaluated without access to $M'(x)$ on the point being queried, yet all other points that are queried are points in which~$M,M'$ do not differ with all but negligible probability.  
Consequently, $\mathcal{C}$ implements a correctable canonicalization.
\end{proof}

\subsection{Instantiating the random oracle }\label{subsec:ro-inst}

We observe that it is possible  to replace the random oracle in the above construction with a  description $p$ of a function chosen randomly from a correlation intractable (CI) hash function family with respect to polytime computable evasive relations (see e.g. \cite{CanettiGH04,CCHLRRW19,PeikertS19,BrakerskiKM20}).    Specifically,  consider the following relation:
\[
R_{G,M} = \{ (w,y)\in(\{0,1\}^\lambda)^2 :  y_i=M(G(w)_i), \  i=1..\lambda\}.
\]
Observe that $R_{G,M}$ is sparse.  In fact, for every fixed $G,M,w$ there is a unique string $y$ such that $R_{G,M}(w,y)$ holds.  It follows that no polynomial-time adversary $A$ can, given $G$, $M$, and a function $h$ drawn from a family $\mathcal H$ of hash functions that are correlation intractable with respect to $R_{G,M}$, find $w$ such that $R_{G,M}(w,h(w))$ holds, except with negligible probability.  This in turn implies soundness for the instantiation based on $G$ and $\mathcal H$.

We note that CI hash families  with respect to relations such as $R_{G,M}$  are known to exist under a variety of standard computational hardness assumptions such  as Decisional Diffie-Hellman,  Learning Parity with Noise  and Learning with Errors (see above references). Furthermore, existence of such CI families implies existence of  pseudorandom generators.

\subsection{Beyond  Evasive Difference Relations}\label{subsec:nonind}

Theorem~\ref{thm:inds} asserts the security of our construction  under the evasive difference relation
$ \approx^{\mathcal E}$.  That is, 
it demonstrates unremovability against adversaries that only modify the model slightly. 

We now strengthen this result to apply also to 
 scenarios in which the thief might modify the watermarked
model $M'$ into a new model $M''$ that is not indistinguishable, but is rather only some 
\emph{distributional approximation} $M'' \approx_{\varepsilon}^{\mathcal{U}} M'$, potentially with a large $\varepsilon <\frac{1}{3}$.   Remarkably, the marked model will remain very close to the original (i.e.,  $M'  \approx^{\mathcal E} M$).

We say that an adversary $\mathcal{R}$ breaks unremovability under
$(\mathcal{U},\varepsilon)$-approximation if, given oracle access to $M'$, it outputs a model $M''$ such that
$M'' \approx_{\varepsilon}^{\mathcal{U}} M'$ and the ownership test rejects with non-negligible
probability.
That is, the watermark is removed while the resulting model remains correct on a
$1-\varepsilon$ fraction of inputs drawn from~$\mathcal{U}$.

Allowing the thief a distributional approximation weakens the conclusion of
Theorem~\ref{thm:inds}, but does not invalidate its core implication:
The distance between each model and its canonicalization will now be~$\Theta(\varepsilon)$.
Furthermore, we get this with essentially the same construction of the proof of ownership.
We now formalize and prove Theorem~\ref{thm:dist-correct}.

We construct~$(\mathcal G,\mathcal W,\mathcal T)$ similarly to Section~\ref{subsec:detailed}. The algorithm~$\mathcal W$ remains unchanged, and the judge~$\mathcal T$ is relaxed as follows.

\paragraph{The judge $\mathcal{T}$.}
Given $(M',w)$, the tester computes $x_1,\ldots,x_\lambda=G(w)$ and $y_1,\ldots,y_\lambda={\cal H}_p(w)$, then samples $k=\Theta(\lambda)$ random indices $S \subseteq [\lambda]$.  It accepts if and only if $M'(x_i)=y_i$ for at least a $\frac{1+\varepsilon}{2}$ fraction of the sampled indices $i\in S$.

\paragraph{}
The proofs of Lemma~\ref{lem:complete} and Lemma~\ref{lem:sound} still hold without modification.

\begin{lemma}
If an adversary $\mathcal{R}$ breaks the unremovability of $(\mathcal G,\mathcal W,\mathcal T)$ for a class $\mathcal{M}$, with respect to~$\approx_\varepsilon^\mathcal{U}$,
then it can be used to construct a correctable canonicalization $C^{(\cdot)}$ with respect to ($\approx^\mathcal{U}_{\left(3\varepsilon+1/\poly(\lambda)\right)}$,$ \approx^{\mathcal E}$)
with comparable time and oracle complexity.
\end{lemma}
\begin{proof}
Since we still have $M' \approx^{\mathcal E} M$, we may define $g_M(x,b)$ and $D(x)$ exactly as in Lemma~\ref{lem:unrem}.
Also as before, we have that for each type~$t\in \{1,2,3\}$,
\[
\Pr_{i\sim [\lambda]}\left(D(x_i)=t\right)
=\Pr_{x\sim \mathcal U}\left(D(x)=t\right) \pm o(1).
\]

By definition of~$M'$ and the independence of~$H$ from~$M$, we also have
\[
\Pr_{i\sim [\lambda]}(M'(x_i)=y_i) = 1,\;\;\;
\Pr_{x\sim \mathcal U}(M'(x)=y_i) = \frac{1}{2}\pm o(1).
\]
Hence, 
\[
\Pr_{i\sim [\lambda]}(M''(x_i)=y_i) = 
\Pr_{x\sim \mathcal U}\left(D(x)=1\right) \cdot \frac{1}{2}
+ 
\Pr_{x\sim \mathcal U}\left(D(x)=2\right) \cdot 1
\pm o(1).
\]

Since $M'' \approx_\varepsilon^\mathcal{U} M'$, the probability that $D(x)=3$ is bounded by~$\varepsilon$ by definition. 
If we assume by contradiction that the probability of~$D(x)=2$ is at least~$2\varepsilon$, then we have that
\[
\Pr_{i\sim [\lambda]}(M''(x_i)=y_i) \geq 
\left(1-3\varepsilon\right) \cdot \frac{1}{2}
+ 
\left(2\varepsilon \right) \cdot 1
\pm o(1)
\geq
\frac{1+\varepsilon}{2} \pm o(1)
\]
and hence the judge accepts~$\mathcal{T}(M'',w)$ in contradiction to $\mathcal{R}$ successfully fooling it.

We conclude that for all but at most~$3\varepsilon+o(1)$ probability over~$\mathcal{U}$, $D(x)=1$; that is, $\mathcal{C}^{M'}(x)$ ignores $M'(x)$ entirely.  
Therefore, the function
\[
g'(x) := g_M(x,0)
\]
is efficiently computable by $\mathcal{R}$, can be evaluated without access to $M'(x)$ on the point~$x$, and is equal to~$M(x)$ on at least a~$(1-3\varepsilon -o(1))$ fraction of the space~$\mathcal{X}$.
\end{proof}

\subsection{Beyond Binary Classifiers}\label{subsec:multiclass}
We note that the above results for binary classifiers extend, with essentially no changes, to
classifiers with any \emph{constant} number of labels. Concretely, let
$M:\mathcal{X}\to\mathcal{Y}$ where $|\mathcal{Y}|=q$ for a fixed constant $q>2$.

\paragraph{Reducing to a binary projection.}
Fix any map $\phi:\mathcal{Y}\to\{0,1\}$.
Consider the induced binary classifier $\phi\circ M:\mathcal{X}\to\{0,1\}$.
All definitions and constructions in this section apply to $\phi\circ M$ verbatim, and therefore
any successful remover for the corresponding proof of ownership implies (via our dichotomy) the
existence of a self-corrector for a function close to $\phi\circ M$.

It remains to explain how to implement the binary construction as a transformation of the original
multi-class model $M$.
In our watermarking procedure, the transformed model $M'$ differs from $M$ only on a designated
watermark set $X_w\subseteq\mathcal{X}$, and on each $x\in X_w$ the transformation enforces a
desired output bit $b\in\{0,1\}$.
In the multi-class setting we proceed identically, except that to enforce bit $b$ we output a fixed
label in $\phi^{-1}(b)$.
Formally, assume $\phi^{-1}(0)$ and $\phi^{-1}(1)$ are both nonempty (otherwise $\phi\circ M$ is
constant and the discussion is trivial), and fix any $y_0\in\phi^{-1}(0)$ and
$y_1\in\phi^{-1}(1)$.
Define the watermarked model $M'$ by
\[
M'(x)\;=\;
\begin{cases}
y_b & \text{if } x\in X_w \text{ and the watermark prescribes bit } b,\\
M(x) & \text{otherwise}.
\end{cases}
\]
The tester $\mathcal{T}$ applies the same binary test as before to the oracle
$\phi(M'(\cdot))$. Soundness, quality preservation, and unremovability  reduce
immediately to the corresponding statements for $\phi\circ M$.

\paragraph{From $q=O(1)$ labels to $\ell=O(1)$ bits.}
When $q$ is constant, we may fix any injective encoding $\mathrm{Enc}:\mathcal{Y}\to\{0,1\}^{\ell}$
with $\ell=\lceil\log_2 q\rceil=O(1)$, and for each $i\in[\ell]$ define the projection
$\phi_i:\mathcal{Y}\to\{0,1\}$ by $\phi_i(y) = (\mathrm{Enc}(y))_i$.
Applying the above reduction with $\phi=\phi_i$, we obtain that for each $i$,
breaking ownership for the induced protocol implies self-correctability (in the sense of this
section) of a function close to $\phi_i\circ M$.

Moreover, if each bit $\phi_i\circ M$ is self-correctable with error at most $\varepsilon$, then so
is $M$ itself with error $O(\varepsilon)$ (since $\ell=O(1)$): on input $x$, run the correctors for
each $\phi_i\circ M$ to obtain a candidate bit-string $b\in\{0,1\}^{\ell}$, and output any label
$y\in\mathcal{Y}$ minimizing the Hamming distance between $\mathrm{Enc}(y)$ and $b$ (this search is
constant-time since $|\mathcal{Y}|=O(1)$). A union bound over $i\in[\ell]$ gives overall error at
most $\ell\varepsilon=O(\varepsilon)$.

Finally, we emphasize that this reduction relies on the classifier-specific requirement that the
thief produce a model that agrees \emph{exactly} with the target model on almost all inputs.
For generative models, where exact equality is not the relevant notion of similarity, one needs a
different framework.

\section{Discussion and Open Problems}\label{sec:discuss}
This work introduced a formal framework for studying \emph{Model Ownership}: the question of whether the ownership of a machine learning model can be provably established after the model has been released.  We proved a fundamental dichotomy under initial basic settings: a classifier admits an unremovable proof of ownership if and only if it is not close to a self-correctable function, in the black-box setting.  
Several natural extensions and challenges remain open, both conceptually and technically.

\paragraph{Beyond Classification.}
Our results focus on binary classifiers, which we then extend to general classifiers and functions, where model similarity and correctness can be defined pointwise over the input domain.  
However, extending the theory to \emph{generative models} with \emph{natural similarity measures} poses a deeper challenge: for such models, the notion of similarity between functions is inherently ambiguous.  
Two text or image-generating models may produce outputs that are perceptually indistinguishable but far apart under any formal metric on strings or pixels.  
Thus, defining a meaningful functional similarity relation, and hence a meaningful notion of self-correction, is itself an open problem.  
Furthermore, in the generative setting, one can also watermark the \emph{outputs} rather than the model's functionality or weights.  
While output watermarking may provide practical ownership guarantees, it belongs to a fundamentally different regime, as it marks data rather than computation.  
Bridging these perspectives requires new formal definitions that can connect functional and distributional similarity.

\paragraph{Owner-Side Noticeable Noise.}
Our formal results for unremovability maintain a strict asymmetry: while Section~\ref{subsec:nonind} allows the \emph{thief} to introduce noticeable distributional noise ($\approx_{\varepsilon}^{\mathcal{U}}$), our dichotomy requires that the \emph{model owner} relies exclusively on evasive difference-based similarity ($\approx^{\mathcal{E}}$). 
In practice, an owner may be willing to noticeably alter their model before publication---paying a noticeable yet small accuracy degradation cost---to protect the model's ownership.
It remains open whether expanding the owner's noise budget fundamentally alters the theoretical barriers established by our dichotomy.

\paragraph{White-Box Setting.}
In the white-box setting, the adversary gains access to the complete description of the model, including its parameters or source code. Once the entire code of $M'$ is visible, the idea of ``self-correction'', which assumes the existence of hidden or inaccessible regions of the function, becomes much harder to define. Establishing a version of our dichotomy that holds in the white-box world, perhaps relying on notions of obfuscation~\cite{BarakGoldreichImpagliazzoRudichSahaiVadhanYang2001,BCP13,ABGSZ13} to reduce the white-box case to the black-box one, or stronger cryptographic assumptions, remains an intriguing open problem.

\subsection*{Acknowledgments}
OZ's research is partially supported by a grant from the Israeli Ministry of Science and Technology.

\bibliographystyle{plain}
\bibliography{ref}
\end{document}